\documentclass{article}

\usepackage{microtype}
\usepackage{graphicx}
\usepackage{subfigure}
\usepackage{booktabs}
\usepackage[english]{babel}

\usepackage{natbib}
\usepackage{hyperref}

\usepackage[accepted]{icml2024}

\usepackage{amsmath}
\usepackage{amssymb}
\usepackage{mathtools}
\usepackage{amsthm}
\usepackage{bbm}
\usepackage[capitalize,noabbrev]{cleveref}
\usepackage{caption}
\usepackage{thmtools}
\usepackage{bm}

\theoremstyle{plain}
\newtheorem{theorem}{Theorem}[section]

\newtheorem{example}{Example}
\theoremstyle{definition}
\newtheorem{definition}[theorem]{Definition}

\theoremstyle{remark}

\declaretheoremstyle[%
  spaceabove=-6pt,%
  spacebelow=6pt,%
  headfont=\normalfont\itshape,%
  postheadspace=1em,%
  qed=\qedsymbol%
]{mystyle}

\DeclareMathOperator{\Id}{\mathbbm{1}} 
\NewDocumentCommand{\grad}{e{_^}}{%
  \mathop{}\!
  \nabla
  \IfValueT{#1}{_{\!#1}}
  \IfValueT{#2}{^{#2}}
}
\DeclareMathOperator{\E}{\mathop{\mathbb{E}}}
\DeclareMathOperator{\Var}{\mathop{\text{Var}}}
\newcommand{\WMC}{\text{WMC}} 

\newcommand{\digit}[1]{\vcenter{\hbox{\includegraphics[height=10pt]{#1}}}}
\newcommand{\var}{\text{var}}

\usepackage[textsize=tiny]{todonotes}

\icmltitlerunning{On the Hardness of Probabilistic Neurosymbolic Learning}

\begin{document}

\twocolumn[
\icmltitle{On the Hardness of Probabilistic Neurosymbolic Learning}

\icmlsetsymbol{equal}{*}

\begin{icmlauthorlist}
\icmlauthor{Jaron Maene}{kul}
\icmlauthor{Vincent Derkinderen}{kul}
\icmlauthor{Luc De Raedt}{kul,orebro}
\end{icmlauthorlist}

\icmlaffiliation{kul}{KU Leuven, Department of Computer Science, Leuven, Belgium.}
\icmlaffiliation{orebro}{\"Orebro University, Centre for Applied Autonomous Sensor Systems, \"Orebro, Sweden}

\icmlcorrespondingauthor{Jaron Maene}{jaron.maene@kuleuven.be}

\icmlkeywords{Neurosymbolic AI, Weighted Model Counting, Gradient Estimation}

\vskip 0.3in
]

\printAffiliationsAndNotice{}

\begin{abstract}
 The limitations of purely neural learning have sparked an interest in probabilistic neurosymbolic models, which combine neural networks with probabilistic logical reasoning. 
 As these neurosymbolic models are trained with gradient descent, we study the complexity of differentiating probabilistic reasoning.
 We prove that although approximating these gradients is intractable in general, it becomes tractable during training. Furthermore, we introduce \emph{WeightME}, an unbiased gradient estimator based on model sampling. Under mild assumptions, WeightME approximates the gradient with probabilistic guarantees using a logarithmic number of calls to a SAT solver.
 Lastly, we evaluate the necessity of these guarantees on the gradient. 
 Our experiments indicate that the existing biased approximations indeed struggle to optimize even when exact solving is still feasible.
\end{abstract}

\section{Introduction}

Neurosymbolic artificial intelligence aims to combine the strengths of neural and symbolic methods into a single unified framework~\cite{garcez_neural-symbolic_2019, hitzler_neuro-symbolic_2022, marra_statistical_2024}. 
One prominent strain of neurosymbolic models combines neural networks with probabilistic reasoning~\cite{manhaeve_deepproblog_2018, xu_semantic_2018, yang_neurasp_2021, ahmed_semantic_2023}.
These models are state-of-the-art but suffer from scalability issues due to the \#P-hard nature of probabilistic inference.
Probabilistic inference has been extensively studied on probabilistic graphical models~\citep{koller_probabilistic_2009} and model counting~\cite{chakraborty_chapter_2021}.
However, neurosymbolic learning deviates from regular probabilistic inference due to its learning aspect: the goal is to efficiently obtain useful gradients rather than to compute the exact probabilities.
We therefore investigate the \emph{approximation of gradients for probabilistic inference} in its own right and prove several positive and negative results.

In Section~\ref{sec:wmc_grad}, we show that the intractability of probabilistic inference also implies that it is impossible to approximate gradients with probabilistic guarantees in polynomial time (unless RP = NP, where RP is randomized polynomial time). 
However, in Section~\ref{sec:intractable_sampling}, we prove that the gradient approximation problem \emph{becomes tractable during training} when the neural network outputs converge to binary values. On the negative side, we find that this tractable region can become unreachable for large problems.

In Section~\ref{sec:wms}, we investigate more general-purpose gradient estimators with strong guarantees by allowing access to a SAT oracle. We introduce the \emph{Weighted Model Estimator (WeightME)}, a novel gradient estimator for probabilistic inference which relies on weighted model sampling. As opposed to existing approximations for neurosymbolic learning \cite{huang_scallop_2021, manhaeve_approximate_2021, ahmed_semantic_2022, li_softened_2022, van_krieken_-nesi_2023, verreet_explain_2023}, WeightME is unbiased and probably approximately correct under mild assumptions. Furthermore, WeightME only needs a logarithmic number of SAT invocations in the number of variables.

Finally, Sections~\ref{sec:approximating_wmc} and~\ref{sec:experiments} provide a comprehensive overview and evaluation of existing approximation techniques. Our results indicate that they have difficulties optimizing benchmarks that can still easily be solved exactly. This suggests that principled methods are warranted if we want to apply probabilistic neurosymbolic optimization to more complex reasoning tasks.

\section{Weighted Model Counting}

Probabilistic inference can be reduced to \emph{weighted model counting} (WMC).
WMC has been called ``an assembly language for probabilistic reasoning"~\citep{belle_hashing-based_2015} as a wide range of probabilistic models can be cast into WMC~\cite{chavira_probabilistic_2008,domshlak_probabilistic_2007,suciu_probabilistic_2011,holtzen_scaling_2020,derkinderen_semirings_2024}.
Bayesian networks are a notable example where WMC solvers are state-of-the-art for exact inference~\cite{chavira_probabilistic_2008}. 
Importantly, WMC also underlies  probabilistic inference in NeSy frameworks \cite{manhaeve_deepproblog_2018, xu_semantic_2018,huang_scallop_2021,ahmed_semantic_2022}.

We briefly introduce propositional logic and WMC.
Propositional \emph{variables} are denoted with lowercase letters (e.g. $x$ or $y$). A \emph{literal} is a variable $x$ or its negation $\neg x$.  A propositional \emph{formula} $\phi$ combines logical variables with the usual connectives: negation $\neg \phi$, conjunction $\phi_1 \land \phi_2$, and disjunction $\phi_1 \lor \phi_2$. 
A \emph{clause} is a disjunction of literals. A \emph{CNF} formula is a conjunction of clauses. A \emph{DNF} formula is a disjunction of conjunctions of literals. We write $\var(\phi)$ for the number of variables in a formula $\phi$. We use $\phi \mid x$ to denote the formula $\phi$ conditioned on $x$ being true, i.e. every occurrence of $x$ (resp. $\neg x$) in $\phi$ is replaced by true (resp. false). 

An \emph{interpretation} $I$ is a set of literals representing an instantiation of the variables. We say that a variable $x$ is true (resp. false) in the interpretation $I$ when $x \in I$ (resp. $\neg x \in I$). An interpretation $I$ is a \emph{model} of the formula $\phi$, denoted as $I \models \phi$, when $\phi$ is satisfied under the truth assignment of $I$.
The \emph{satisfiability problem (SAT)} asks whether a formula has at least one model, while \emph{model counting (\#SAT)} asks how many models a formula has. The \emph{weighted model count} (WMC) is a weighted sum over the models.

\begin{definition}[Weighted Model Count]
    Given a propositional logic formula $\phi$ and a weight function $w$ that maps every literal to a real number, the weighted model count is
    \begin{equation}\label{eq:wmc}
        \WMC(\phi;w) = \sum_{I: I \models \phi} \prod_{l \in I} w(l)
    \end{equation}
\end{definition}

Both \#SAT and WMC are \#P-complete problems \cite{valiant_complexity_1979}, and even approximating them with probabilistic guarantees is NP-hard \cite{roth_hardness_1996}.

Due to our probabilistic focus, we only consider weights corresponding to Bernoulli distributions.
More concretely, we assume $w(x) \in [0, 1]$ and $w(x) = 1 - w(\neg x)$ for every variable $x$.
Consequently, we have a probability distribution over interpretations: $P(I;w) = \prod_{l\in I} w(l)$. 

\begin{example}
    Consider the formula $\phi = (a \lor b) \land (\neg b \lor c)$ with weights $w(a)=0.5$, $w(b)=0.1$, and $w(c)=0.25$. This formula has four models: $\{a, \neg b, \neg c\}$, $\{\neg a, b, c\}$, $\{a, \neg b, c\}$, and $\{a, b, c\}$. When we sum the probabilities of these models we get the weighted model count.
    \begin{equation*}
        \WMC(\phi, w) = 0.3375 + 0.0125 + 0.1125 + 0.0125 = 0.475
    \end{equation*}
\end{example}

We omit the weight function $w$ when it is clear from context. In Appendix~\ref{app:wmc}, we further discuss how to handle categorical distributions and unweighted variables in WMC and explain how inference on a Bayesian network reduces to WMC.

In the neurosymbolic context, the weights $w$ are the probabilities produced by a neural network. Typically, these weights will be random at initialization and get closer to zero or one during training as the neural network becomes more confident in its predictions.

\section{From WMC to $\grad \WMC$}\label{sec:wmc_grad}
Probabilistic neurosymbolic methods optimize the WMC by iteratively updating the weights $w$ with gradient descent. So just like inference in probabilistic models can be reduced to WMC, learning in probabilistic neurosymbolic models can be reduced to taking the gradient of the WMC.

\[\grad_w \WMC(\phi; w) = \left[ \frac {\partial \WMC(\phi; w)} {\partial w(x_1)},..., \frac {\partial \WMC(\phi; w)} {\partial w(x_{\var(\phi)})} \right]^\top \]  
This WMC gradient is the core focus of our work, and as the next theorem shows, is closely related to the WMC itself.

\begin{theorem}\label{thm:equiv}
    Computing the partial derivative of a WMC problem is reducible to WMC problems, and vice versa.
\end{theorem}
\begin{proof}
Both directions follow from the decomposition
\[\WMC(\phi) = w(x) \WMC(\phi\mid x) + w(\neg x) \WMC(\phi \mid \neg x)\]
The insight here is that $\WMC(\phi \mid x)$ and $\WMC(\phi \mid \neg x)$ do not have a gradient for $w(x)$, implying that
\[
\frac {\partial \WMC(\phi, w)} {\partial w(x)} = \WMC(\phi \mid x) 
 - \WMC(\phi \mid \neg x)
\]
For the other direction, we introduce a dummy variable $t$. 
\[\WMC(\phi, w) =  \frac{ \partial \WMC(\phi \land t, w)} {\partial w(t)} \qedhere\]
\end{proof}

Theorem~\ref{thm:equiv} allows us to study the $\grad \WMC$ problem using existing results on $\WMC$.
Notably, it follows immediately that computing the exact gradients is \#P-complete, just like for WMC. Gradient descent does not necessarily need exact gradients, however. We therefore investigate approximations of $\grad \WMC$. In particular, we focus on approximations that 1) are unbiased and 2) have probabilistic guarantees. Informally, this means that the approximation 1) is correct in expectation and 2) has a high probability of being close to the true gradient. This probabilistic guarantee is formalized as an $(\epsilon, \delta)$-approximation.

\begin{definition}\label{dfn:eps_delta_approximation}
    An estimator $\hat y$ is an $(\epsilon, \delta)$-approximation for $y$ when $P\left(\lvert (y - \hat y) / y \rvert > \epsilon \right) \leq \delta$.
\end{definition}

The $(\epsilon, \delta)$-approximation enforces a probabilistic bound on the relative error: the probability of having a relative error larger than $\epsilon$ is at most $\delta$. This guarantee is also known as probably approximately correct (PAC). 
Unfortunately, $(\epsilon, \delta)$-approximations of the $\WMC$ are still hard, and due to Theorem~\ref{thm:equiv} the same can be said for derivatives.

\begin{theorem}\label{thm:approx_intract}
     Computing an $(\epsilon, \delta)$-approximation of the partial derivative $\partial \WMC(\phi) / \partial w(x)$ is NP-hard.
\end{theorem}
\begin{proof}
    Follows from Theorem 3.2 of \citet{roth_hardness_1996} and Theorem~\ref{thm:equiv}.
\end{proof}

It is worth reflecting on why an $(\epsilon, \delta)$-approximation of the gradient is relevant as opposed to the variance of the gradient, which is more commonly discussed.
The variance $\text{Var}[\hat y] = \E[(y - \hat y)^2]$ measures the squared error instead of the relative error. 
However, a low squared error is not meaningful when the WMC and $\grad \WMC$ are near zero. In Appendix~\ref{app:variance}, we demonstrate how existing gradient estimators for $\grad \WMC$ achieve low variance by simply returning zero gradients. An $(\epsilon, \delta)$-approximation does not suffer from this problem as it measures the relative instead of the squared error.

\section{The (In)tractability of Sampling}\label{sec:intractable_sampling}

One way to get an unbiased $(\epsilon, \delta)$-approximation for the WMC is by sampling. Indeed, the WMC can be seen as the expectation that a random interpretation is a model. 
\begin{equation}\label{eq:exp_wmc}
    \WMC(\phi; w) = \E_{I \sim P(I;w)} [\Id(I \models \phi)]
\end{equation}
Here, $\Id(\cdot)$ denotes the indicator function.
Equation~\ref{eq:exp_wmc} leads to a straightforward Monte Carlo approximation, which we call \textit{interpretation sampling}. To the best of our knowledge, all existing unbiased estimators for $\grad \WMC$ are based on interpretation sampling. We can obtain gradients with conventional techniques such as the score function estimator (SFE), also known as REINFORCE \cite{sutton_policy_1999}. 
\begin{equation*}
\grad_w \WMC(\phi;w) = \E_{I \sim P(I;w)} \Id(I \models \phi) \grad_w \log P(I;w)
\end{equation*}
Alternatively, we can use the decomposition of Theorem~\ref{thm:equiv}. This corresponds to the IndeCateR estimator, which is a Rao-Blackwellization of the SFE \cite{de_smet_differentiable_2023}. 
\begin{align*}
\frac {\partial \WMC(\phi, w)} {\partial w(x)} =& \ \E_{I \sim P(I\mid x;w)} \Id(I \models \phi) \\
 &- \E_{I \sim P(I\mid \neg x;w)} \Id(I \models \phi)
\end{align*}

From Theorem~\ref{thm:approx_intract}, we already know that interpretation sampling cannot give a $(\epsilon, \delta)$-approximation in polynomial time. Indeed, \citet{karp_monte-carlo_1989} showed that $c(\epsilon, \delta) / \WMC(\phi)$ samples are required for an $(\epsilon, \delta)$-approximation of $\WMC(\phi)$, where $c(\epsilon, \delta)$ is in the order of $\epsilon^{-2} \log \frac 2 \delta$.
The problematic part here is $1/\WMC(\phi)$ because the WMC can decrease exponentially in the number of variables\footnote{To see this, consider the case where $\phi$ has a single model and all weights are $w(x) = \frac 1 2$, so that $\WMC(\phi) = 2^{-var(\phi)}$.}.

\begin{figure}
    \centering
    \includegraphics[width=\columnwidth]{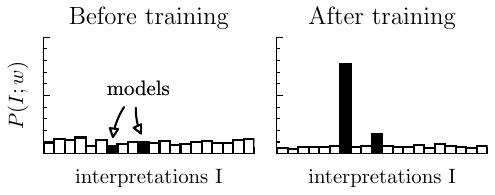}
    \caption{When sampling from the distribution of interpretations, we need to hit a model to obtain a gradient. (left): At initialization, the distribution over interpretations is fairly uniform and the probability that interpretation sampling finds a model is vanishingly small. Model sampling avoids this by sampling directly from the models. (right): When the neural network becomes more confident in its predictions, it becomes easier to sample a model.}
    \label{fig:sketch}
\end{figure}

\subsection{Example of Tractability}

In the neurosymbolic setting, interpretation sampling can \emph{become tractable during training}. 
This occurs when the neural network becomes more confident, i.e. when the weights $w(x)$ approach zero or one.
We illustrate this phenomenon on the pedagogical MNIST-addition task. 

In the MNIST-addition task, a digit classifier is trained with distant supervision.
There are no labels for the individual MNIST images, but only on the sums of two numbers represented by images.
For example, the input $\digit{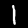}\digit{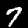}\digit{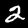}\digit{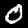}+\digit{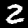}\digit{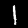}\digit{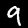}\digit{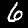}$ has the label 3916. 
We refer to \citet{manhaeve_deepproblog_2018} for more details on the experimental setup.

Figure~\ref{fig:mnist_cosine} displays the first epoch of training with exact inference on MNIST-addition with 4 digits. At every iteration, we estimate the gradient with the SFE and compare it with the true gradient. Initially, the neural network is random, and estimating the gradients is infeasible. But at about 700 training iterations, there is a transition after which the sampled gradient becomes a faithful approximation.

\begin{figure}
    \centering
    \centerline{\includegraphics[width=\columnwidth]{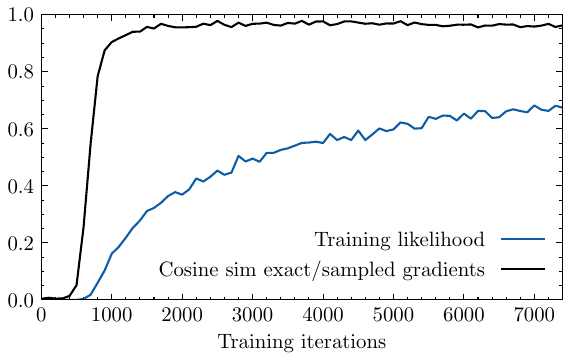}}
    \caption{First epoch of training on 4-digit MNIST-addition with exact inference. We also plot the cosine similarity between the exact gradients and the sampled gradients from the SFE.
    Results are averaged over 10 seeds.}\label{fig:mnist_cosine}
\end{figure}

\subsection{Tractability from Training}

We formalize the above example using implicants. An \emph{implicant} of a formula $\phi$ is a conjunction of literals that models a subset of the models of $\phi$.
Now, we prove that an $(\epsilon, \delta)$-approximation of the gradient becomes tractable when a single implicant dominates the WMC. 
\begin{theorem}\label{thm:tract_sampling}
    The partial derivative of $\WMC(\phi)$ with respect to $w(x)$ admits a polynomial time $(\epsilon, \delta)$-approximation when there exists a implicant $\pi$ such that $x \in \pi$ and $\WMC(\pi) \geq \WMC(\phi \mid\neg x) + c(\epsilon, \delta)$.
\end{theorem}
\begin{proof}
    See Appendix~\ref{prf:tract_sampling}.
\end{proof}

The assumption that an implicant dominates the probability mass towards the end of training has been observed before empirically \cite{manhaeve_approximate_2021} and was recently proven to hold for all minima of a WMC problem \cite{van_krieken_independence_2024}.

As our definition of an $(\epsilon, \delta)$-approximation pertains to a single partial derivative, some partial derivatives in the gradient may be tractable while others are not. 
For Theorem~\ref{thm:tract_sampling} to apply, we need $\WMC(\phi \mid \neg x)$ to converge to zero as $\WMC(\pi)$ converges to one. 
This will generally be the case, unless a different implicant $\pi'$ partly covers both $\pi$ and $(\pi \mid x) \land \neg x$. Or more informally, when tractability is not achieved at convergence, it is precisely because the value of this literal does not matter for the WMC. For example, this can occur because $\pi$ is not prime, meaning $\pi \mid x$ is still an implicant.

\subsection{Tractability from Concept Supervision}

The practical relevance of Theorem~\ref{thm:tract_sampling} depends on how close the neural network must be to convergence before we end up in the tractable region.
We formalize this as follows.

\begin{definition}
    Consider a formula $\phi$ with weights $w$ and $M$ the most probable model. We say that $\phi$ is \emph{$\tau$-supervised} when
    $\tau = \#\{ x \mid x \in M \text{ and } w(x) > 1/2 \} $.
\end{definition}

In other words, the weights are $\tau$-supervised when the neural network classifies $\tau$ of the $\var(\phi)$ weights confidently with respect to a model $M$.

\begin{theorem}\label{thm:intract_concept}
To ensure a polynomial time $(\epsilon, \delta)$-approximation of the partial derivative $\partial \WMC(\phi) / \partial w(x)$ for a $\tau$-supervised formula $\phi$ with interpretation sampling, it is required that $var(\phi) - c'(\epsilon, \delta) \leq \tau$.
\end{theorem}
\begin{proof}
See Appendix~\ref{prf:instract_concept}.
\end{proof}

Theorem~\ref{thm:intract_concept} says that in the worst case, the number of variables that should be $\tau$-supervised is close to the total number of variables. More precisely, the percentage of variables that are allowed to be misclassified by the neural network decreases as $1/var(\phi)$. Firstly, this implies that the tractability of Theorem~\ref{thm:tract_sampling} might not always be reachable in practice.
Secondly, Theorem~\ref{thm:intract_concept} suggests that concept supervision cannot entirely alleviate the need for approximate inference. Concept supervision uses some direct supervision on the weights of logical variables, instead of training from scratch. For example, in the MNIST-addition example, some images get an individual label (e.g. $\digit{images/9.png}=9$).

\section{Weighted Model Sampling}\label{sec:wms}

When the WMC is too small for interpretation sampling, we need to look elsewhere to obtain gradients. 
The past decades have seen considerable progress in approximate model counting, despite its NP-hardness \cite{chakraborty_chapter_2021}.
Leveraging this progress for $\grad \WMC$ could push the scalability of gradient approximations with guarantees.

\citet{stockmeyer_complexity_1983} proved that model counting with PAC guarantees is possible using a polynomial number of SAT calls. More recently, \citet{chakraborty_algorithmic_2016} sharpened this to a logarithmic number of SAT calls in the number of variables.
The state-of-the-art in approximate unweighted model counting implements this with hash-based methods \cite{soos_bird_2019}. In short, they use hash functions to randomly partition the space of interpretations and count models within these partitions. 

Hash-based methods could be applied immediately to obtain gradients using the decomposition of Theorem~\ref{thm:equiv}. However, that would require $2 \var(\phi)$ calls to the approximate counter to calculate a single gradient, as every partial derivative is computed separately. Furthermore, the subtraction of two $(\epsilon, \delta)$-approximations weakens the provided guarantees.

We can do better by instead relying on model sampling. \emph{Weighted model sampling} (WMS) is the task of sampling models from a weighted formula $\phi$ such that the probability of selecting a model $M$ is $P(M) / \WMC(\phi)$.
Using WMS, we introduce the following estimator. 

\begin{definition}
The \emph{weighted model estimator} (WeightME) is defined as
    \[ \frac{\partial \log \WMC(\phi)} {\partial w(x)} = \E_M \left[ \frac{\Id(x \in M)} {w(x)} - \frac {\Id(x \not\in M)} {w(\neg x)} \right] \]
\end{definition}

With $\E_M$ we denote an expectation over the models of $\phi$. WeightME is unbiased and achieves probabilistic guarantees using only a constant number of samples.

\begin{theorem}\label{thm:wms_unbiased}
WeightME is an unbiased estimator for $\partial \log \WMC(\phi) / \partial w(x)$ when $w(x) \in (0, 1)$.
\end{theorem}
\begin{proof}
    See Appendix~\ref{prf:wms_unbiased}.
\end{proof}

\begin{theorem}\label{thm:wms}
Given $\WMC(\phi)$, WeightME can $(\epsilon, \delta)$-approximate the partial derivative $\partial \WMC(\phi) / \partial w(x)$ using a constant number of weighted model samples when there is a constant $\lambda > 0$ such that $\lvert P(x \mid \phi) - w(x) \rvert > \lambda$ and $w(x) \in ({0}, {1})$.

\end{theorem}
\begin{proof}
The proof relies on the observation that
\[
\frac {\partial \WMC(\phi)} {\partial w(x)} = \frac {\WMC(\phi)} {w(\neg x)} \cdot \left( \frac{\E_M [\Id(x \in M)]} {w(x)} - 1 \right)
\]
With $c(\epsilon, \delta)/\lambda^2$ weighted model samples, the above is an $(\epsilon, \delta)$-approximation.
The full proof is in Appendix~\ref{prf:wms}.
\end{proof}

This last theorem has several implications.
\begin{enumerate}
    \item The required number of model samples for an $(\epsilon, \delta)$-approximation does not increase with  formula size.
    Indeed, the condition that $w(x)$ and $\WMC(\phi)$ are dependent does not rely on the formula size. So in contrast to Theorem~\ref{thm:tract_sampling}, Theorem~\ref{thm:wms} scales to large formulas where $\WMC(\phi)$ is very small.
    \item Theorems \ref{thm:tract_sampling} and \ref{thm:wms} can be seen as complementary. The variance of IndeCateR is zero when the weights are binary, while WeightME has a low variance when the weights are close to 1/2.
    \item The condition in Theorem~\ref{thm:wms} of having $\WMC(\phi)$ can be dropped if we are satisfied with $\grad \log \WMC$ instead of $\grad \WMC$. This is a reasonable assumption as the WMC tends to be optimized with a negative log-likelihood loss.
\end{enumerate}

\subsection{Approximate model sampling}

\begin{table*}[t]
    \centering
        \caption{Classification of all considered approximation methods according to whether they are unbiased, deterministic, are learned from data, what formulas they support, and their complexity. For the complexity, we denote the number of clauses as $c$, and the number of samples as $s$. We assume the clause length is bounded, and the number of variables is linear in the size of the clause.}
    \label{tab:summary}
    
    \begin{tabular}{l c c c c c}
        \toprule
        Method & Unbiased & Deterministic & Learned & Formula & Time Complexity \\\midrule
        Interpretation sampling & Y & N & N & Any & $cs$ \\
        Weighted model sampling & Y & N & N & Any & NP-hard\\\midrule

        Hash-based & Y & N & N & Any & NP-hard \\
        Fuzzy t-norms & N & Y & N & Any & $c$ \\
        MPE / $k$-best & N & Y & N & Any & NP-hard  \\
        Unweighted model sampling & N & N & N & Any & NP-hard  \\
        Neural\#DNF & N & Y & Y & DNF & $c$  \\
        A-NeSI & N & Y & Y & Any & $c$ \\
        Collapsed sampling & Y & N & N & Any & $cs$ \\
        Bounded inference & / & Y & N & Any & $c$  \\
        Semantic strengthening & N & Y & N & Any & $c^2$  \\
        Straight-trough estimator & N & N & N & Any & $cs$  \\
        Gumbel-Softmax & N & N & N & Any & $cs$  \\
        I-MLE & N & N & N & Any & NP-hard \\
        \bottomrule
    \end{tabular}
\end{table*}

The question remains how to actually sample the models.
WMS and WMC are polynomially inter-reducible, so exact WMS is also \#P-hard \cite{jerrum_random_1986}. Scalable WMS methods hence resort to approximations and do not sample exactly according to the weighted model distribution $P(M)/\WMC(\phi)$. Crucially, Theorem~\ref{thm:wms} can still apply when the WMS samples are $(\epsilon, \delta)$-approximate (see Appendix~\ref{app:approx_wms}). A single WeightME gradient is therefore possible with only a logarithmic number of calls to a SAT oracle \cite{chakraborty_algorithmic_2016}.

Just as for WMC, approximate WMS with PAC guarantees is implemented with hash-based techniques
\cite{soos_tinted_2020}. The unweighted variant of model sampling has received more attention than WMS \cite{chakraborty_distribution-aware_2014}. However, it is possible to convert weighted into unweighted problems \cite{chakraborty_weighted_2015}, which could make it possible to leverage the progress in unweighted sampling for gradient approximation.

Biased WMS approximations have the promise to scale further, but lack guarantees. \citet{golia_designing_2021} investigated this trade-off, by using sampling testers on biased approximations. This led them to propose a performant solver that samples from a distribution that does not differ from the true distribution on statistical tests.
Markov-Chain Monte-Carlo (MCMC) techniques have also been proposed for model sampling \cite{ermon_uniform_2012}. Unfortunately, combinatorial problems face an exponential mixing time of the Markov chains. \citet{li_softened_2022} addressed this problem by using projection techniques from SMT solvers.

A promising recent development is the sampling of models using neural approximations. \citet{van_krieken_-nesi_2023} connected the theory of GFlowNets \cite{bengio_gflownet_2023} to model sampling. However, to the best of our knowledge, this has not yet been realized on practical WMC problems.

\section{Biased WMC}\label{sec:approximating_wmc}

Various approximate inference methods exist that trade in guarantees for better scalability, and are often not NP-hard.
Theory does not always align with practice, and it is hence not implausible that some of these could be competitive.

We give a comprehensive overview of relevant approximations for WMC but do not aim to be exhaustive. 
Instead, we focus on the most notable and common approaches. 
We summarize all methods in Table~\ref{tab:summary}.

\subsection{Biased inference}

The following methods compute a biased yet differentiable approximation for the WMC. 

\textbf{$\mathbf{k}$-Best} approximates a formula using a DNF, containing the $k$ implicants with the highest probability \cite{kimmig_efficient_2008, manhaeve_approximate_2021, huang_scallop_2021}. This works well when the probability mass of a formula can be captured by only a small number of implicants, for which exact WMC is then feasible. For $k=1$, $k$-best coincides with the most probable explanation (MPE).

The probability mass of the $k$-best implicants can overlap heavily, so \textbf{$\mathbf{k}$-optimal} greedily searches for the $k$-DNF with the maximum total probability \cite{renkens_k-optimal_2012}. Finding these implicants reduces to iteratively solving weighted MaxSAT problems \cite{renkens_explanation-based_2014}.

\textbf{Uniform model sampling} samples models uniformly without considering the weights. As with $k$-best, these models can be used as a lower bound for the true WMC. \citet{verreet_explain_2023} argue for this approach to maximize the diversity of samples during training.

\textbf{Fuzzy t-norms}~are arguably the most common neurosymbolic semantics. They replace the logical (Boolean) operations with continuous generalizations~\cite{badreddine_logic_2022, van_krieken_analyzing_2022}. For example, the product t-norm computes conjunction as $w(x \land y) = w(x) \cdot w(y)$ and disjunction as $w(x \lor y) = 1- w(\neg x)\cdot w(\neg y)$. Crucially, fuzzy semantics has linear complexity in the size of the propositional formula.
T-norms can also be used as an approximation of probabilistic semantics. For example, the product t-norm computes the WMC of a CNF under the assumption that all clauses are independent.

\textbf{Neural approximations}. As neural networks are universal approximators, recent works have proposed using them to approximate probabilistic inference. From the algebraic view, \citet{zuidberg_dos_martires_neural_2021} introduces the neural semiring, where a neural network learns the algebraic operations for conjunction and disjunction. \citet{abboud_learning_2020} realize this with graph neural networks. The authors limit the scope to DNFs, motivated by the tractability of this special case.

\subsection{Biased gradient estimation}

Several gradient estimators have been proposed to differentiate through sampling and hence can be applied to Equation~\ref{eq:exp_wmc}. Notably, several biased estimators can still give a gradient when the WMC is tiny.

When estimating the gradients of an expectation of the form $\E_{x\sim P(x)} f(x)$, biased estimators typically assume that $f$ is continuous. As such, we need a continuous relaxation for $\Id(I \models \phi)$. The straightforward solution we consider is to use fuzzy semantics to determine whether an interpretation is a model of the formula.

The \textbf{straight-through estimator} (STE) replaces the sample with its probability during backpropagation \cite{bengio_estimating_2013}.
\textbf{Gumbel-Softmax} estimates gradients by replacing the Categorical distribution with the differentiable Gumbel-Softmax distribution \cite{jang_categorical_2016, maddison_concrete_2016}. 

\textbf{Implicit Maximum Likelihood Estimation} (I-MLE) is a biased gradient estimator designed for combinatorial problems \cite{niepert_implicit_2021}. I-MLE approximates the gradient using a perturb-and-MAP approach: the weights are perturbed by adding some noise from a distribution such as Gumbel, and on these weights the most probable model is computed. I-MLE can hence be seen as an alternative to $k$-best rooted in perturbation-based implicit differentation.

\subsection{Hybrid methods}

Many of the approximation methods can be combined with each other, or with exact solving. We highlight some examples of this.

\textbf{Collapsed sampling}~uses an exact solver up to a time or memory limit, after which the remaining components are estimated using interpretation sampling.
Similar to exact solving, the choice of variable ordering greatly impacts the effectiveness of this method \cite{friedman_approximate_2018}.

\textbf{Semantic strengthening} combines exact solving with fuzzy t-norms \cite{ahmed_semantic_2022}. It uses mutual information to determine which clause pairs violate the independence assumption most, and hence benefit from exact compilation. The remaining conjunctions are computed with a t-norm.

\section{Experiments}\label{sec:experiments}

\begin{table*}[t]
\centering
\caption{Benchmarks for gradient estimation at initialization. For each solver-benchmark combination, we list the average cosine similarity between the approximate and true gradient, as well as the standard deviation. Higher is better. We use ${-}$~to indicate that at least one instance timed out.}
\label{tab:results}

\begin{tabular}{l c c c c}
\toprule
 & MCC2021 & MCC2022 & MCC2023 & ROAD-R \\
\midrule
Number of Instances & 120 & 93 & 61 & 100 \\
\midrule
WeightME (k=100) & \textbf{0.784 ± 0.248} &\textbf{0.721 ± 0.259} & \textbf{0.821 ± 0.172} & \textbf{0.955 ± 0.009} \\
Unweighted model sampling (k=100) & - & 0.578 ± 0.332 & - & 0.946 ± 0.014 \\
MPE & - & - & - & 0.609 ± 0.042 \\
$k$-Optimal ($k$=100) & - & - & - & 0.691 ± 0.034 \\
\midrule
Product t-norm& 0.643 ± 0.259 & 0.592 ± 0.188 & 0.537 ± 0.310 & 0.918 ± 0.005 \\
G\"odel t-norm& 0.092 ± 0.153 & 0.107 ± 0.119 & 0.103 ± 0.136 & 0.191 ± 0.094 \\
Straight-through estimator (s=10)& 0.195 ± 0.217 & 0.167 ± 0.202 & 0.309 ± 0.187 & 0.537 ± 0.078 \\
Gumbel-Softmax estimator (s=10, $\tau$=2)& 0.584 ± 0.263 & 0.501 ± 0.187 & 0.510 ± 0.293 & 0.897 ± 0.020 \\
SFE (s=10k) & 0.035 ± 0.173 & 0.010 ± 0.093 & 0.000 ± 0.000 & 0.006 ± 0.064 \\
Semantic strengthening ($\kappa$=100) & - & - & - & 0.895 ± 0.019 \\
\bottomrule
\end{tabular}
\end{table*}

It is clear that many methods exist to approximate WMC gradients, so the question arises as to which of the methods are appropriate in practice.
For this reason, we evaluate the gradients of the various methods on a set of challenging WMC benchmarks.\footnote{The code to replicate these experiments can be found at \url{https://github.com/jjcmoon/hardness-nesy}.}

\paragraph{Benchmarks} The model counting competition (MCC) is an annual competition on (weighted) model counting \cite{fichte_model_2021}. We take the benchmarks of the last three competitions (2021, 2023, and 2023) and take the instances that are probabilistic and can be solved exactly by state-of-the-art solvers \cite{lagniez_improved_2017, golia_designing_2021}. As an easier benchmark, we also include the logical formula from the ROAD-R dataset \cite{giunchiglia_road-r_2023}, which imposes constraints on the object detection of self-driving cars. All benchmarks are CNF formulas. The weights are initialized with a Gaussian distribution with a mean of 1/2.

\paragraph{Setup} All methods were executed on the same machine with an Intel Xeon E5-2690 CPU and used PyTorch to compute the gradients. WeightME was implemented using CMSGen \cite{golia_designing_2021} for WMS, as it was found to be the most scalable WMS solver available. For MPE and $k$-optimal we used the EvalMaxSAT solver \cite{avellaneda_short_2020}. The SFE was implemented with the Reinforce-Leave-One-Out baseline \cite{kool_buy_2019}. All approximate methods got a timeout of 5 minutes per gradient. The true gradients were computed as ground truths using the d4 knowledge compiler \cite{lagniez_improved_2017}, which did not get a timeout.

\subsection{Gradients at initialization}

In the first experiment, we empirically validate which approximation methods succeed at the gradient estimation task at initialization.

\paragraph{Quality} To evaluate the quality of the gradients, we compute the cosine similarity between the exact and approximate gradients on our set of benchmarks. 
Table~\ref{tab:results} summarizes the results.
WeightME attains the best results, both compared to polynomial and NP-hard methods. For the polynomial methods, the product t-norm and Gumbel-Softmax perform best. The G\"odel t-norm performs weak, as by design it gives zero gradients to all but one variable. The SFE usually fails to sample a model, which is why it performs very poorly on these benchmarks. IndeCateR is not included in the results, as it suffers even harder from this problem. Unweighted model sampling performs similarly to WeightME on the smaller ROAD-R benchmark, but falls behind on more challenging problems.

\paragraph{Scalability} In Figure~\ref{fig:cactus}, we look at the runtime of the various approaches. None of the NP-hard approximation methods manages to solve all benchmarks within the time limit (5 minutes/instance). Surprisingly, none of the tested MaxSAT or approximate WMS solvers could fully match d4. The polynomial methods scale quite well as expected, except for semantic strengthening, which quickly becomes infeasible when the number of clauses is high. 

\begin{figure}[h!]
    \centering
    \centerline{\includegraphics[width=\columnwidth]{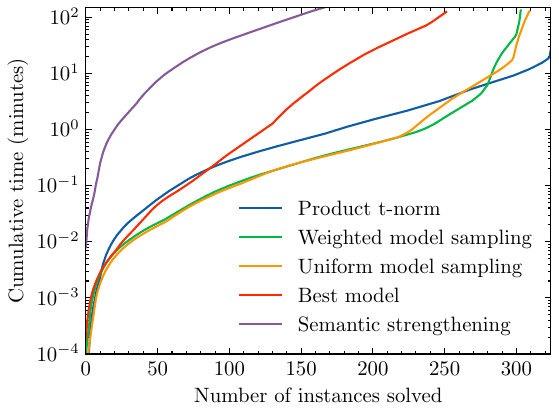}}
    \caption{Cumulative runtimes on all the MCC instances. Omitted methods achieve similar performance to the Product t-norm.}
    \label{fig:cactus}
\end{figure}

\subsection{Optimization}

\begin{figure*}[h!]
    \centering
    \centerline{\includegraphics[width=0.9\textwidth]{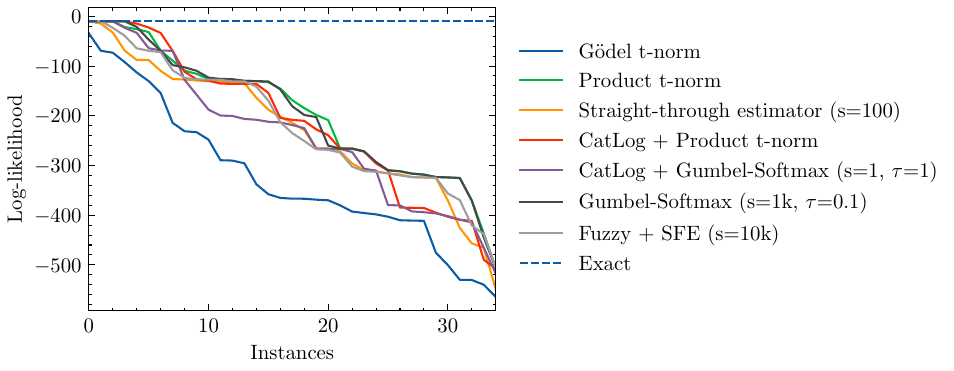}}
    \caption{Maximum log-likelihood achieved by the various biased gradient approximations, sorted from best to worst. The benchmarks are 33 easy instances from the Model Counting Competitions. Higher is better.}
    \label{fig:optim}
\end{figure*}

The previous experiments evaluate the quality of a single gradient in isolation. However, the real question is whether the approximate gradients suffice to optimize. The setup in Table~\ref{tab:results} penalizes unbiased methods with high gradient variance. So in a second experiment, the task is to optimize the log-likelihood of the formulas. We can see this as a necessary (though not necessarily sufficient) requirement for probabilistic neurosymbolic learning. The NP-hard methods that explicitly find a model can trivially achieve this task and are omitted from this experiment.

In Figure~\ref{fig:optim}, we let polynomial approximation methods optimize the log-likelihood of a formula on an easier subset of the MCC benchmarks. These benchmarks have fewer than 1000 variables and are well below the limit of state-of-the-art exact solvers. We plot the best achieved loss, within a maximum of 10000 iterations. Our results indicate that none of the tested approaches can consistently optimize. The results in Figure~\ref{fig:optim} also include some novel baselines which are described in Appendix~\ref{app:baselines}. In Appendix~\ref{app:experiments}, we further validate that adding concept supervision does not alleviate the optimization problem.

\section{Related work}

The estimation of gradients is a well-studied problem in machine learning \cite{mohamed_monte_2020}. Existing works have studied the gradient estimation of categorical distributions \cite{jang_categorical_2016, maddison_concrete_2016, de_smet_differentiable_2023}, while we are the first to focus on the gradients of $\WMC$. \citet{van_krieken_analyzing_2022} analyze the gradients of fuzzy semantics in neurosymbolic learning. On the other hand, we target neurosymbolic learning with probabilistic semantics.

Related to our work, \citet{niepert_implicit_2021} propose a gradient estimator for black-box combinatorial solvers relying on the most probable model, which is both biased and harder to compute than approximate weighted model samples. \citet{verreet_explain_2023} introduce a neurosymbolic optimization method using unweighted model samples, motivated by increasing the sample diversity. However, unweighted sampling provides weaker guarantees while being just as hard as weighted sampling.

Considerable progress has been made on approximating WMC and WMS \cite{hutchison_new_2005, gogate_samplesearch_2011, ermon_uniform_2012, chakraborty_distribution-aware_2014, chakraborty_algorithmic_2016, soos_bird_2019, golia_designing_2021, soos_tinted_2020}.
We focus on the gradients of WMC in a learning setting instead of approximating the WMC itself.

\section{Limitations}

As with any empirical study, the results of Section~\ref{sec:experiments} are influenced by the choice of benchmarks, and might not generalize to all neurosymbolic tasks. 
Our work only considers propositional logic, while some neurosymbolic systems target the more expressive first-order logic. First-order neurosymbolic systems usually end up grounding their theory such that the propositional case studied here remains relevant. Inference for weighted first-order model counting is much harder still \cite{gribkoff_understanding_2014}, and hence the need to approximate is even more pertinent. We also do not consider DNFs, which in contrast to CNFs admit a tractable $(\epsilon, \delta)$-approximation \cite{karp_monte-carlo_1989}.

\section{Conclusion}

We studied the gradient estimation of probabilistic reasoning by connecting this problem to weighted model counting. 
This allowed us to prove several results on the intractability, and prove how gradient estimation for neurosymbolic learning becomes tractable during training. Next, we contributed a general-purpose gradient estimator that builds on the progress in approximate counting and sampling. We showed that a constant number of weighted model samples is sufficient to achieve strong guarantees. 

Finally, we turned our attention to existing approximation methods. Our experiments suggest that none of the polynomial methods can consistently optimize a formula. In contrast, existing NP-hard approximations typically struggle to scale, while lacking guarantees. Potential further work includes improving our understanding of the interaction between approximating the weighted model samples and the PAC guarantee of WeightME and expanding our analysis from propositional to first-order weighted model counting.

\section*{Acknowledgements}

This research received funding from the Flemish Government (AI Research Program), the Flanders Research Foundation (FWO) under project G097720N, the KU Leuven Research Fund (C14/18/062) and TAILOR, a project from the EU Horizon 2020 research and innovation program under GA No 952215. Luc De Raedt is also supported by the Wallenberg AI, Autonomous Systems and Software Program (WASP) funded by the Knut and Alice Wallenberg Foundation.

We are grateful to Lennert De Smet for his feedback on a draft of this paper, and to Pedro Zuidberg Dos Martires for a helpful discussion.

\section*{Impact Statement}

This paper presents work whose goal is to advance the field of neurosymbolic artificial intelligence. There are many potential societal consequences of our work, none of which we feel must be specifically highlighted here.

\bibliography{references}

\begin{thebibliography}{58}
\providecommand{\natexlab}[1]{#1}
\providecommand{\url}[1]{\texttt{#1}}
\expandafter\ifx\csname urlstyle\endcsname\relax
  \providecommand{\doi}[1]{doi: #1}\else
  \providecommand{\doi}{doi: \begingroup \urlstyle{rm}\Url}\fi

\bibitem[Abboud et~al.(2020)Abboud, Ceylan, and Lukasiewicz]{abboud_learning_2020}
Abboud, R., Ceylan, I., and Lukasiewicz, T.
\newblock Learning to {Reason}: {Leveraging} {Neural} {Networks} for {Approximate} {DNF} {Counting}.
\newblock \emph{Proceedings of the AAAI Conference on Artificial Intelligence}, 34\penalty0 (04):\penalty0 3097--3104, April 2020.
\newblock ISSN 2374-3468.
\newblock \doi{10.1609/aaai.v34i04.5705}.
\newblock Number: 04.

\bibitem[Ahmed et~al.(2022)Ahmed, Teso, Chang, Van~den Broeck, and Vergari]{ahmed_semantic_2022}
Ahmed, K., Teso, S., Chang, K.-W., Van~den Broeck, G., and Vergari, A.
\newblock Semantic {Probabilistic} {Layers} for {Neuro}-{Symbolic} {Learning}.
\newblock In \emph{Advances in {Neural} {Information} {Processing} {Systems}}, May 2022.

\bibitem[Ahmed et~al.(2023)Ahmed, Chang, and Van~den Broeck]{ahmed_semantic_2023}
Ahmed, K., Chang, K.-W., and Van~den Broeck, G.
\newblock Semantic {Strengthening} of {Neuro}-{Symbolic} {Learning}.
\newblock In \emph{Proceedings of {The} 26th {International} {Conference} on {Artificial} {Intelligence} and {Statistics}}, pp.\  10252--10261. PMLR, April 2023.
\newblock ISSN: 2640-3498.

\bibitem[Avellaneda(2020)]{avellaneda_short_2020}
Avellaneda, F.
\newblock A short description of the solver {EvalMaxSAT}.
\newblock \emph{MaxSAT Evaluation}, 8, 2020.

\bibitem[Badreddine et~al.(2022)Badreddine, Garcez, Serafini, and Spranger]{badreddine_logic_2022}
Badreddine, S., Garcez, A.~d., Serafini, L., and Spranger, M.
\newblock Logic {Tensor} {Networks}.
\newblock \emph{Artificial Intelligence}, 303:\penalty0 103649, February 2022.
\newblock ISSN 00043702.
\newblock \doi{10.1016/j.artint.2021.103649}.

\bibitem[Belle et~al.(2015)Belle, Van~den Broeck, and Passerini]{belle_hashing-based_2015}
Belle, V., Van~den Broeck, G., and Passerini, A.
\newblock Hashing-based approximate probabilistic inference in hybrid domains.
\newblock In \emph{Proceedings of the 31st {Conference} on {Uncertainty} in {Artificial} {Intelligence} ({UAI})}, pp.\  141--150. AUAI PRESS, 2015.

\bibitem[Bengio et~al.(2013)Bengio, Léonard, and Courville]{bengio_estimating_2013}
Bengio, Y., Léonard, N., and Courville, A.
\newblock Estimating or {Propagating} {Gradients} {Through} {Stochastic} {Neurons} for {Conditional} {Computation}, August 2013.

\bibitem[Bengio et~al.(2023)Bengio, Lahlou, Deleu, Hu, Tiwari, and Bengio]{bengio_gflownet_2023}
Bengio, Y., Lahlou, S., Deleu, T., Hu, E.~J., Tiwari, M., and Bengio, E.
\newblock {GFlowNet} {Foundations}.
\newblock \emph{Journal of Machine Learning Research}, 24\penalty0 (210):\penalty0 1--55, 2023.
\newblock ISSN 1533-7928.

\bibitem[Chakraborty et~al.(2014)Chakraborty, Fremont, Meel, Seshia, and Vardi]{chakraborty_distribution-aware_2014}
Chakraborty, S., Fremont, D., Meel, K., Seshia, S., and Vardi, M.
\newblock Distribution-{Aware} {Sampling} and {Weighted} {Model} {Counting} for {SAT}.
\newblock \emph{Proceedings of the AAAI Conference on Artificial Intelligence}, 28\penalty0 (1), June 2014.
\newblock ISSN 2374-3468.
\newblock \doi{10.1609/aaai.v28i1.8990}.
\newblock Number: 1.

\bibitem[Chakraborty et~al.(2015)Chakraborty, Fried, Meel, and Vardi]{chakraborty_weighted_2015}
Chakraborty, S., Fried, D., Meel, K.~S., and Vardi, M.~Y.
\newblock From {Weighted} to {Unweighted} {Model} {Counting}.
\newblock In \emph{{IJCAI}}, pp.\  689--695, 2015.

\bibitem[Chakraborty et~al.(2016)Chakraborty, Meel, and Vardi]{chakraborty_algorithmic_2016}
Chakraborty, S., Meel, K.~S., and Vardi, M.~Y.
\newblock Algorithmic improvements in approximate counting for probabilistic inference: from linear to logarithmic {SAT} calls.
\newblock In \emph{Proceedings of the {Twenty}-{Fifth} {International} {Joint} {Conference} on {Artificial} {Intelligence}}, {IJCAI}'16, pp.\  3569--3576, New York, New York, USA, July 2016. AAAI Press.
\newblock ISBN 978-1-57735-770-4.

\bibitem[Chakraborty et~al.(2021)Chakraborty, Meel, and Vardi]{chakraborty_chapter_2021}
Chakraborty, S., Meel, K.~S., and Vardi, M.~Y.
\newblock Chapter 26. {Approximate} {Model} {Counting}.
\newblock In \emph{Handbook of {Satisfiability}}, pp.\  1015--1045. IOS Press, 2021.
\newblock \doi{10.3233/FAIA201010}.

\bibitem[Chavira \& Darwiche(2008)Chavira and Darwiche]{chavira_probabilistic_2008}
Chavira, M. and Darwiche, A.
\newblock On probabilistic inference by weighted model counting.
\newblock \emph{Artificial Intelligence}, 172\penalty0 (6):\penalty0 772--799, April 2008.
\newblock ISSN 0004-3702.
\newblock \doi{10.1016/j.artint.2007.11.002}.

\bibitem[De~Smet et~al.(2023)De~Smet, Sansone, and Zuidberg Dos~Martires]{de_smet_differentiable_2023}
De~Smet, L., Sansone, E., and Zuidberg Dos~Martires, P.
\newblock Differentiable {Sampling} of {Categorical} {Distributions} {Using} the {CatLog}-{Derivative} {Trick}.
\newblock In \emph{Advances in {Neural} {Information} {Processing} {Systems}}, November 2023.

\bibitem[Derkinderen et~al.(2024)Derkinderen, Manhaeve, Zuidberg Dos~Martires, and De~Raedt]{derkinderen_semirings_2024}
Derkinderen, V., Manhaeve, R., Zuidberg Dos~Martires, P., and De~Raedt, L.
\newblock Semirings for probabilistic and neuro-symbolic logic programming.
\newblock \emph{International Journal of Approximate Reasoning}, pp.\  109130, January 2024.
\newblock ISSN 0888-613X.
\newblock \doi{10.1016/j.ijar.2024.109130}.

\bibitem[Domshlak \& Hoffmann(2007)Domshlak and Hoffmann]{domshlak_probabilistic_2007}
Domshlak, C. and Hoffmann, J.
\newblock Probabilistic {Planning} via {Heuristic} {Forward} {Search} and {Weighted} {Model} {Counting}.
\newblock \emph{Journal of Artificial Intelligence Research}, 30:\penalty0 565--620, December 2007.
\newblock ISSN 1076-9757.
\newblock \doi{10.1613/jair.2289}.

\bibitem[Ermon et~al.(2012)Ermon, Gomes, and Selman]{ermon_uniform_2012}
Ermon, S., Gomes, C., and Selman, B.
\newblock Uniform solution sampling using a constraint solver as an oracle.
\newblock In \emph{Proceedings of the {Twenty}-{Eighth} {Conference} on {Uncertainty} in {Artificial} {Intelligence}}, {UAI}'12, pp.\  255--264, Arlington, Virginia, USA, August 2012. AUAI Press.
\newblock ISBN 978-0-9749039-8-9.

\bibitem[Fichte et~al.(2021)Fichte, Hecher, and Hamiti]{fichte_model_2021}
Fichte, J.~K., Hecher, M., and Hamiti, F.
\newblock The {Model} {Counting} {Competition} 2020.
\newblock \emph{ACM Journal of Experimental Algorithmics}, 26:\penalty0 13:1--13:26, 2021.
\newblock ISSN 1084-6654.
\newblock \doi{10.1145/3459080}.

\bibitem[Friedman \& Van~den Broeck(2018)Friedman and Van~den Broeck]{friedman_approximate_2018}
Friedman, T. and Van~den Broeck, G.
\newblock Approximate {Knowledge} {Compilation} by {Online} {Collapsed} {Importance} {Sampling}.
\newblock In \emph{Advances in {Neural} {Information} {Processing} {Systems}}, volume~31. Curran Associates, Inc., 2018.

\bibitem[Garcez et~al.(2019)Garcez, Gori, Lamb, Serafini, Spranger, and Tran]{garcez_neural-symbolic_2019}
Garcez, A., Gori, M., Lamb, L., Serafini, L., Spranger, M., and Tran, S.
\newblock Neural-{Symbolic} {Computing}: {An} {Effective} {Methodology} for {Principled} {Integration} of {Machine} {Learning} and {Reasoning}.
\newblock \emph{FLAP}, May 2019.

\bibitem[Giunchiglia et~al.(2023)Giunchiglia, Stoian, Khan, Cuzzolin, and Lukasiewicz]{giunchiglia_road-r_2023}
Giunchiglia, E., Stoian, M.~C., Khan, S., Cuzzolin, F., and Lukasiewicz, T.
\newblock {ROAD}-{R}: {The} {Autonomous} {Driving} {Dataset} with {Logical} {Requirements}.
\newblock \emph{Machine Learning}, May 2023.
\newblock ISSN 0885-6125, 1573-0565.
\newblock \doi{10.1007/s10994-023-06322-z}.

\bibitem[Gogate \& Dechter(2011)Gogate and Dechter]{gogate_samplesearch_2011}
Gogate, V. and Dechter, R.
\newblock {SampleSearch}: {Importance} sampling in presence of determinism.
\newblock \emph{Artificial Intelligence}, 175\penalty0 (2):\penalty0 694--729, February 2011.
\newblock ISSN 0004-3702.
\newblock \doi{10.1016/j.artint.2010.10.009}.

\bibitem[Golia et~al.(2021)Golia, Soos, Chakraborty, and Meel]{golia_designing_2021}
Golia, P., Soos, M., Chakraborty, S., and Meel, K.~S.
\newblock Designing {Samplers} is {Easy}: {The} {Boon} of {Testers}.
\newblock In \emph{2021 {Formal} {Methods} in {Computer} {Aided} {Design} ({FMCAD})}, pp.\  222--230, October 2021.
\newblock \doi{10.34727/2021/isbn.978-3-85448-046-4_31}.
\newblock ISSN: 2708-7824.

\bibitem[Gribkoff et~al.(2014)Gribkoff, Van~den Broeck, and Suciu]{gribkoff_understanding_2014}
Gribkoff, E., Van~den Broeck, G., and Suciu, D.
\newblock Understanding the complexity of lifted inference and asymmetric {Weighted} {Model} {Counting}.
\newblock In \emph{Proceedings of the {Thirtieth} {Conference} on {Uncertainty} in {Artificial} {Intelligence}}, {UAI}'14, pp.\  280--289, Arlington, Virginia, USA, July 2014. AUAI Press.
\newblock ISBN 978-0-9749039-1-0.

\bibitem[Hitzler(2022)]{hitzler_neuro-symbolic_2022}
Hitzler, P.
\newblock \emph{Neuro-{Symbolic} {Artificial} {Intelligence}: {The} {State} of the {Art}}.
\newblock IOS Press, January 2022.
\newblock ISBN 978-1-64368-245-7.

\bibitem[Holtzen et~al.(2020)Holtzen, Van~den Broeck, and Millstein]{holtzen_scaling_2020}
Holtzen, S., Van~den Broeck, G., and Millstein, T.
\newblock Scaling exact inference for discrete probabilistic programs.
\newblock \emph{Proceedings of the ACM on Programming Languages}, 4\penalty0 (OOPSLA):\penalty0 140:1--140:31, November 2020.
\newblock \doi{10.1145/3428208}.

\bibitem[Huang et~al.(2021)Huang, Li, Chen, Samel, Naik, Song, and Si]{huang_scallop_2021}
Huang, J., Li, Z., Chen, B., Samel, K., Naik, M., Song, L., and Si, X.
\newblock Scallop: {From} {Probabilistic} {Deductive} {Databases} to {Scalable} {Differentiable} {Reasoning}.
\newblock In \emph{Advances in {Neural} {Information} {Processing} {Systems}}, volume~34, pp.\  25134--25145. Curran Associates, Inc., 2021.

\bibitem[Jang et~al.(2016)Jang, Gu, and Poole]{jang_categorical_2016}
Jang, E., Gu, S., and Poole, B.
\newblock Categorical {Reparameterization} with {Gumbel}-{Softmax}.
\newblock In \emph{International {Conference} on {Learning} {Representations}}, November 2016.

\bibitem[Jerrum et~al.(1986)Jerrum, Valiant, and Vazirani]{jerrum_random_1986}
Jerrum, M.~R., Valiant, L.~G., and Vazirani, V.~V.
\newblock Random generation of combinatorial structures from a uniform distribution.
\newblock \emph{Theoretical Computer Science}, 43:\penalty0 169--188, January 1986.
\newblock ISSN 0304-3975.
\newblock \doi{10.1016/0304-3975(86)90174-X}.

\bibitem[Karp et~al.(1989)Karp, Luby, and Madras]{karp_monte-carlo_1989}
Karp, R.~M., Luby, M., and Madras, N.
\newblock Monte-{Carlo} approximation algorithms for enumeration problems.
\newblock \emph{Journal of Algorithms}, 10\penalty0 (3):\penalty0 429--448, September 1989.
\newblock ISSN 0196-6774.
\newblock \doi{10.1016/0196-6774(89)90038-2}.

\bibitem[Kimmig et~al.(2008)Kimmig, Santos~Costa, Rocha, Demoen, and De~Raedt]{kimmig_efficient_2008}
Kimmig, A., Santos~Costa, V., Rocha, R., Demoen, B., and De~Raedt, L.
\newblock On the {Efficient} {Execution} of {ProbLog} {Programs}.
\newblock In Garcia de~la Banda, M. and Pontelli, E. (eds.), \emph{Logic {Programming}}, Lecture {Notes} in {Computer} {Science}, pp.\  175--189, Berlin, Heidelberg, 2008. Springer.
\newblock ISBN 978-3-540-89982-2.
\newblock \doi{10.1007/978-3-540-89982-2_22}.

\bibitem[Koller \& Friedman(2009)Koller and Friedman]{koller_probabilistic_2009}
Koller, D. and Friedman, N.
\newblock \emph{Probabilistic {Graphical} {Models}: {Principles} and {Techniques}}.
\newblock MIT Press, July 2009.
\newblock ISBN 978-0-262-01319-2.

\bibitem[Kool et~al.(2019)Kool, Hoof, and Welling]{kool_buy_2019}
Kool, W., Hoof, H.~v., and Welling, M.
\newblock Buy 4 {REINFORCE} {Samples}, {Get} a {Baseline} for {Free}!
\newblock April 2019.

\bibitem[Lagniez \& Marquis(2017)Lagniez and Marquis]{lagniez_improved_2017}
Lagniez, J.-M. and Marquis, P.
\newblock An improved decision-{DNNF} compiler.
\newblock In \emph{Proceedings of the 26th {International} {Joint} {Conference} on {Artificial} {Intelligence}}, {IJCAI}'17, pp.\  667--673, Melbourne, Australia, August 2017. AAAI Press.
\newblock ISBN 978-0-9992411-0-3.

\bibitem[Li et~al.(2022)Li, Yao, Chen, Xu, Cao, Ma, and Lü]{li_softened_2022}
Li, Z., Yao, Y., Chen, T., Xu, J., Cao, C., Ma, X., and Lü, J.
\newblock Softened {Symbol} {Grounding} for {Neuro}-symbolic {Systems}.
\newblock In \emph{International {Conference} on {Learning} {Representations}}, September 2022.

\bibitem[Maddison et~al.(2016)Maddison, Mnih, and Teh]{maddison_concrete_2016}
Maddison, C.~J., Mnih, A., and Teh, Y.~W.
\newblock The {Concrete} {Distribution}: {A} {Continuous} {Relaxation} of {Discrete} {Random} {Variables}.
\newblock In \emph{International {Conference} on {Learning} {Representations}}, November 2016.

\bibitem[Manhaeve et~al.(2018)Manhaeve, Dumancic, Kimmig, Demeester, and De~Raedt]{manhaeve_deepproblog_2018}
Manhaeve, R., Dumancic, S., Kimmig, A., Demeester, T., and De~Raedt, L.
\newblock {DeepProbLog}: {Neural} {Probabilistic} {Logic} {Programming}.
\newblock In \emph{Advances in {Neural} {Information} {Processing} {Systems}}, volume~31. Curran Associates, Inc., 2018.

\bibitem[Manhaeve et~al.(2021)Manhaeve, Marra, and De~Raedt]{manhaeve_approximate_2021}
Manhaeve, R., Marra, G., and De~Raedt, L.
\newblock Approximate {Inference} for {Neural} {Probabilistic} {Logic} {Programming}.
\newblock In \emph{Proceedings of the 18th {International} {Conference} on {Principles} of {Knowledge} {Representation} and {Reasoning}}, pp.\  475--486, 2021.

\bibitem[Marra et~al.(2024)Marra, Dumančić, Manhaeve, and De~Raedt]{marra_statistical_2024}
Marra, G., Dumančić, S., Manhaeve, R., and De~Raedt, L.
\newblock From {Statistical} {Relational} to {Neurosymbolic} {Artificial} {Intelligence}: a {Survey}.
\newblock \emph{Artificial Intelligence}, pp.\  104062, January 2024.
\newblock ISSN 0004-3702.
\newblock \doi{10.1016/j.artint.2023.104062}.

\bibitem[Mohamed et~al.(2020)Mohamed, Rosca, Figurnov, and Mnih]{mohamed_monte_2020}
Mohamed, S., Rosca, M., Figurnov, M., and Mnih, A.
\newblock Monte {Carlo} gradient estimation in machine learning.
\newblock \emph{The Journal of Machine Learning Research}, 21\penalty0 (1):\penalty0 132:5183--132:5244, January 2020.
\newblock ISSN 1532-4435.

\bibitem[Niepert et~al.(2021)Niepert, Minervini, and Franceschi]{niepert_implicit_2021}
Niepert, M., Minervini, P., and Franceschi, L.
\newblock Implicit {MLE}: {Backpropagating} {Through} {Discrete} {Exponential} {Family} {Distributions}.
\newblock In \emph{Advances in {Neural} {Information} {Processing} {Systems}}, volume~34, pp.\  14567--14579. Curran Associates, Inc., 2021.

\bibitem[Renkens et~al.(2012)Renkens, Van den~Broeck, and Nijssen]{renkens_k-optimal_2012}
Renkens, J., Van den~Broeck, G., and Nijssen, S.
\newblock k-{Optimal}: a novel approximate inference algorithm for {ProbLog}.
\newblock \emph{Machine Learning}, 89\penalty0 (3):\penalty0 215--231, December 2012.
\newblock ISSN 0885-6125, 1573-0565.
\newblock \doi{10.1007/s10994-012-5304-9}.

\bibitem[Renkens et~al.(2014)Renkens, Kimmig, Van~den Broeck, and De~Raedt]{renkens_explanation-based_2014}
Renkens, J., Kimmig, A., Van~den Broeck, G., and De~Raedt, L.
\newblock Explanation-{Based} {Approximate} {Weighted} {Model} {Counting} for {Probabilistic} {Logics}.
\newblock \emph{Proceedings of the AAAI Conference on Artificial Intelligence}, 28\penalty0 (1), June 2014.
\newblock ISSN 2374-3468.
\newblock \doi{10.1609/aaai.v28i1.9067}.
\newblock Number: 1.

\bibitem[Roth(1996)]{roth_hardness_1996}
Roth, D.
\newblock On the hardness of approximate reasoning.
\newblock \emph{Artificial Intelligence}, 82\penalty0 (1):\penalty0 273--302, April 1996.
\newblock ISSN 0004-3702.
\newblock \doi{10.1016/0004-3702(94)00092-1}.

\bibitem[Soos \& Meel(2019)Soos and Meel]{soos_bird_2019}
Soos, M. and Meel, K.~S.
\newblock {BIRD}: {Engineering} an {Efficient} {CNF}-{XOR} {SAT} {Solver} and {Its} {Applications} to {Approximate} {Model} {Counting}.
\newblock \emph{Proceedings of the AAAI Conference on Artificial Intelligence}, 33\penalty0 (01):\penalty0 1592--1599, July 2019.
\newblock ISSN 2374-3468.
\newblock \doi{10.1609/aaai.v33i01.33011592}.
\newblock Number: 01.

\bibitem[Soos et~al.(2020)Soos, Gocht, and Meel]{soos_tinted_2020}
Soos, M., Gocht, S., and Meel, K.~S.
\newblock Tinted, {Detached}, and {Lazy} {CNF}-{XOR} {Solving} and {Its} {Applications} to {Counting} and {Sampling}.
\newblock In Lahiri, S.~K. and Wang, C. (eds.), \emph{Computer {Aided} {Verification}}, Lecture {Notes} in {Computer} {Science}, pp.\  463--484, Cham, 2020. Springer International Publishing.
\newblock ISBN 978-3-030-53288-8.
\newblock \doi{10.1007/978-3-030-53288-8_22}.

\bibitem[Stockmeyer(1983)]{stockmeyer_complexity_1983}
Stockmeyer, L.
\newblock The complexity of approximate counting.
\newblock In \emph{Proceedings of the fifteenth annual {ACM} symposium on {Theory} of computing}, {STOC} '83, pp.\  118--126, New York, NY, USA, December 1983. Association for Computing Machinery.
\newblock ISBN 978-0-89791-099-6.
\newblock \doi{10.1145/800061.808740}.

\bibitem[Suciu et~al.(2011)Suciu, Olteanu, Ré, and Koch]{suciu_probabilistic_2011}
Suciu, D., Olteanu, D., Ré, C., and Koch, C.
\newblock \emph{Probabilistic {Databases}}.
\newblock Synthesis {Lectures} on {Data} {Management}. Springer International Publishing, Cham, 2011.
\newblock ISBN 978-3-031-00751-4 978-3-031-01879-4.
\newblock \doi{10.1007/978-3-031-01879-4}.

\bibitem[Sutton et~al.(1999)Sutton, McAllester, Singh, and Mansour]{sutton_policy_1999}
Sutton, R.~S., McAllester, D., Singh, S., and Mansour, Y.
\newblock Policy {Gradient} {Methods} for {Reinforcement} {Learning} with {Function} {Approximation}.
\newblock In \emph{Advances in {Neural} {Information} {Processing} {Systems}}, volume~12. MIT Press, 1999.

\bibitem[Valiant(1979)]{valiant_complexity_1979}
Valiant, L.~G.
\newblock The {Complexity} of {Enumeration} and {Reliability} {Problems}.
\newblock \emph{SIAM Journal on Computing}, 8\penalty0 (3):\penalty0 410--421, August 1979.
\newblock ISSN 0097-5397.
\newblock \doi{10.1137/0208032}.
\newblock Publisher: Society for Industrial and Applied Mathematics.

\bibitem[van Krieken et~al.(2022)van Krieken, Acar, and van Harmelen]{van_krieken_analyzing_2022}
van Krieken, E., Acar, E., and van Harmelen, F.
\newblock Analyzing {Differentiable} {Fuzzy} {Logic} {Operators}.
\newblock \emph{Artificial Intelligence}, 302:\penalty0 103602, January 2022.
\newblock ISSN 0004-3702.
\newblock \doi{10.1016/j.artint.2021.103602}.

\bibitem[van Krieken et~al.(2023)van Krieken, Thanapalasingam, Tomczak, Harmelen, and Teije]{van_krieken_-nesi_2023}
van Krieken, E., Thanapalasingam, T., Tomczak, J.~M., Harmelen, F.~V., and Teije, A.~T.
\newblock A-{NeSI}: {A} {Scalable} {Approximate} {Method} for {Probabilistic} {Neurosymbolic} {Inference}.
\newblock In \emph{Advances in {Neural} {Information} {Processing} {Systems}}, November 2023.

\bibitem[van Krieken et~al.(2024)van Krieken, Minervini, Ponti, and Vergari]{van_krieken_independence_2024}
van Krieken, E., Minervini, P., Ponti, E.~M., and Vergari, A.
\newblock On the {Independence} {Assumption} in {Neurosymbolic} {Learning}.
\newblock In \emph{International {Conference} on {Machine} {Learning}}, April 2024.

\bibitem[Verreet et~al.(2023)Verreet, De~Smet, and Sansone]{verreet_explain_2023}
Verreet, V., De~Smet, L., and Sansone, E.
\newblock {EXPLAIN}, {AGREE} and {LEARN}: {A} {Recipe} for {Scalable} {Neural}-{Symbolic} {Learning}, 2023.

\bibitem[Wei \& Selman(2005)Wei and Selman]{hutchison_new_2005}
Wei, W. and Selman, B.
\newblock A {New} {Approach} to {Model} {Counting}.
\newblock In Hutchison, D., Kanade, T., Kittler, J., Kleinberg, J.~M., Mattern, F., Mitchell, J.~C., Naor, M., Nierstrasz, O., Pandu~Rangan, C., Steffen, B., Sudan, M., Terzopoulos, D., Tygar, D., Vardi, M.~Y., Weikum, G., Bacchus, F., and Walsh, T. (eds.), \emph{Theory and {Applications} of {Satisfiability} {Testing}}, volume 3569, pp.\  324--339. Springer Berlin Heidelberg, Berlin, Heidelberg, 2005.
\newblock ISBN 978-3-540-26276-3 978-3-540-31679-4.
\newblock \doi{10.1007/11499107_24}.
\newblock Series Title: Lecture Notes in Computer Science.

\bibitem[Xu et~al.(2018)Xu, Zhang, Friedman, Liang, and Van~den Broeck]{xu_semantic_2018}
Xu, J., Zhang, Z., Friedman, T., Liang, Y., and Van~den Broeck, G.
\newblock A {Semantic} {Loss} {Function} for {Deep} {Learning} with {Symbolic} {Knowledge}.
\newblock In \emph{Proceedings of the 35th {International} {Conference} on {Machine} {Learning}}, pp.\  5502--5511. PMLR, July 2018.
\newblock ISSN: 2640-3498.

\bibitem[Yang et~al.(2021)Yang, Ishay, and Lee]{yang_neurasp_2021}
Yang, Z., Ishay, A., and Lee, J.
\newblock {NeurASP}: embracing neural networks into answer set programming.
\newblock In \emph{Proceedings of the {Twenty}-{Ninth} {International} {Joint} {Conference} on {Artificial} {Intelligence}}, {IJCAI}'20, pp.\  1755--1762, Yokohama, Yokohama, Japan, January 2021.
\newblock ISBN 978-0-9992411-6-5.

\bibitem[Zuidberg Dos~Martires(2021)]{zuidberg_dos_martires_neural_2021}
Zuidberg Dos~Martires, P.
\newblock Neural {Semirings}.
\newblock 2021.

\end{thebibliography}

\newpage
\appendix
\onecolumn

\section{Example of Baysian Networks as Weighted Model Counting}\label{app:wmc}

As a small example, consider the context of a Bayesian network with a conditional probability $P(B|A)$ and independent probabilistic fact $P(A)$, over probabilistic Boolean variables $A$ and $B$. This can be encoded into a WMC formula $\phi'$:
\begin{equation}
    b \iff \big( a \land \theta_{b|a} \big) \lor \big(\neg a \land \theta_{b|\neg a})  \text{,}
\end{equation}
with the weights of $P(B|A)$ attached to newly introduced variables $\theta_{b|a}$ and $\theta_{b|\neg a}$. As is common, we can set $w(b) = w(\neg b) = 1$. The probability of $B$, for instance, is then $P(B) = WMC(\phi' \land b)$.

We limit our study to WMC problems where the weights correspond to a Bernoulli distribution. The example above, where $w(b) = w(\neg b) = 1$, can be encoded into this setting by instead using $w(b) = w(\neg b) = 0.5$ and normalizing the WMC by a factor $2^{u}$, with $u$ the number of variables whose original weights were $1$. 
We do remark that this affects the tractability of additive $(\epsilon, \delta)$-bounds.

Categorical variables can be modeled using exclusively Boolean variables. For example, without using $w( \cdot )=1$, suppose $A$ takes values $a_1$, $a_2$ or $a_3$:
\begin{equation}
    \big(a_1 \iff \theta_{a_1} \big) \land
    \big(a_2 \iff \neg a_1 \land \theta_{a_2|\neg a_1} \big) \land
    \big(a_3 \iff \neg a_1 \land \neg a_2 \big)
\end{equation}
The weights of $a_i$ are $0.5$ (using the method above), and the other weights are as expected: $w(\theta_{a_1}) = P(a_1)$, $w(\neg \theta_{a_1}) = 1 - w(\theta_{a_1})$, $w(\theta_{a_2 | \neg a_1}) = \frac{P(A = a_2)}{P(A \neq a_1)}$ and $w(\neg \theta_{a_2 | \neg a_1}) = 1 - w(\theta_{a_2|\neg a_1})$.

\section{Variance of IndeCateR}\label{app:variance}

Consider the single-sample IndeCateR estimator \cite{de_smet_differentiable_2023} for $\partial \WMC(\phi) / \partial w(x)$.
\[C = \Id(I_0 \land x \models \phi) - \Id(I_1 \land \neg x \models \phi)  \] 

By considering Theorem~\ref{thm:equiv}, it follows that $\Var[C] \leq \E[C^2] = (\WMC(\phi \mid x) - \WMC(\phi \mid \neg x))^2 \leq 1$.
Moreover, this variance is maximized when the weights $w$ are binary. If we also know that the weights are bounded by $[t, 1-t]$ and $\phi$ has a single model, we achieve the much stronger result that $ \Var[C] \leq (1-t)^{\var(\phi)}$. Although IndeCateR has vanishingly small variance on formulas with a single model, IndeCateR will almost never sample a model here and simply return zero gradients. This motivates that variance is not an ideal lens for looking at the quality of a gradient estimator for WMC.

\section{Proofs}\label{app:proofs}

\subsection*{Theorem \ref{thm:tract_sampling}}\label{prf:tract_sampling}
\begin{proof}
We need to show that the following bound can be computed in polynomial time:

$$ P(\lvert (y - \hat y) / y \rvert > \epsilon)
\leq \delta \ \  \text{ where }  y = \frac {\partial \WMC(\phi, w)} {\partial w(x)} = \WMC(\phi \mid x) - \WMC(\phi \mid \neg x) $$

As the theorem assumes that there is an implicant $\pi$ with $x \in \pi$, it follows that:

\begin{align*}
    y &= \WMC(\phi \mid x) - \WMC(\phi \mid \neg x) \\
    &\geq \WMC(\pi \mid x) - \WMC(\phi \mid \neg x) \\
    &\geq \WMC(\pi) - \WMC(\phi \mid \neg x) \\
    &\geq c(\epsilon, \delta)
\end{align*}

In the last step, we use the assumed inequality of the theorem. So in summary, the assumption of the theorem implies that we know that $y \geq c(\epsilon, \delta)$. We also know that $\hat y \in [-1, 1]$.
It follows that we can estimate it in polynomial time with sampling. To get a concrete expression for $c(\epsilon, \delta)$ and derive a bound on the number samples, we can use e.g. Hoefdding bounds.
\end{proof}

\subsection*{Theorem~\ref{thm:intract_concept}}\label{prf:instract_concept}

\begin{proof}
Suppose $\phi$ only has a single model. This means that $y = \partial \WMC(\phi) / \partial w(x) = \WMC(M \mid x) = p(M) / w(x)$, assuming w.l.o.g. that $x\in M$. To ensure tractability by interpretation sampling, we need $y > c(\epsilon, \delta)$ \cite{karp_monte-carlo_1989}. Because $\phi$ is $\tau$-supervised, we have $P(M) \leq 2^{\tau - \var(\phi)}$. When we combine these facts, we get that $w(x) \cdot c(\epsilon, \delta) < 2^{\tau - \var(\phi)}$, which is equivalent to $\log_2 w(x) c(\epsilon, \delta) + \var(\phi) < \tau$. To simplify the theorem, we set $c'(\epsilon, \delta) = -\log_2 w(x)c(\epsilon, \delta)$, which brings use to $\var(\phi) - c'(\epsilon, \delta) < \tau$.
\end{proof}

\subsection*{Theorem~\ref{thm:wms_unbiased}}\label{prf:wms_unbiased}
\begin{proof} We write $X$ for the random variable of WeightME. 
\begin{align*}
    \E_M[X] &=  \frac 1 {w(x)} \E_M[\Id(x \in M)] - \frac 1 {w(\neg x)} \E_M[ \Id(x \not\in M)] \\
    &= \Big( \sum_M \frac {\Id(x \in M) P(M)} {w(x) \WMC(\phi)} \Big) - \Big( \sum_M  \frac { \Id(x \not\in M) P(M)} {w(\neg x) \WMC(\phi)} \Big) \\
    &= \frac {\WMC(\phi \land x)} {w(x) \WMC(\phi)} - \frac { \WMC(\phi \land \neg x) } {w(\neg x) \WMC(\phi)} \\
    &= \frac {\WMC(\phi \mid x)} {\WMC(\phi)} - \frac { \WMC(\phi \mid \neg x) } { \WMC(\phi)} \\
    &= \frac {1} {\WMC(\phi)} \cdot \frac {\partial \WMC(\phi)} {\partial w(x)} = \frac {\partial \log \WMC(\phi)} {\partial w(x)} \qedhere
\end{align*}
\end{proof}

\subsection*{Theorem~\ref{thm:wms}}\label{prf:wms}
\begin{proof}

We introduce an random variable $T$ for $\Id(x \in M)$, where $\E_M[T] = \WMC(\phi \land x) / \WMC(\phi)$. Using $T$, single-sample WeightME can be written as 
\begin{align*}
    X &= \frac{1}{w(x)} T - \frac{1}{w(\neg x)} (1-T) \\
    &= T \left( \frac 1 {w(x)} + \frac 1 {w(\neg x)} \right) - \frac{1}{w(\neg x)} \\
    &= \frac {T} {w(x) w(\neg x)} - \frac{1}{w(\neg x)}
\end{align*}

Remember that the theorem assumes $w(x) \neq 0$ and $w(\neg x) \neq 0$. 
Next, we look for an $(\epsilon, \delta)$-approximation for WeightME, i.e. we again need to prove that
\[ P(\lvert (y - \hat y) / y \rvert > \epsilon)
\leq \delta \ \  \text{ where }  y = \frac {\partial \WMC(\phi, w)} {\partial w(x)} = \WMC(\phi \mid x) - \WMC(\phi \mid \neg x) \]

We can write WeightME with $s$ weighted model samples as the estimator $X' = \sum_{i=1}^s \frac{X_i} s$ where $\E[X'] = \E[X] = y / \WMC(\phi)$ and $\hat y = \WMC(\phi) X'$.

\begin{align*}
P\left(\left\lvert \frac{\hat y - y}{y} \right\rvert > \epsilon \right) 
&= P\left(\left\lvert \frac{X' - \E[X']}{\E[X']} \right\rvert > \epsilon \right) \\
&= P\left(\left\lvert T' - \E[T'] \right\rvert > \epsilon \lvert \E[T'] - w(x) \rvert \right) \\
&\leq 2 \exp(-2 s \epsilon^2 (\E[T'] - w(x))^2) 
\end{align*}
In the last step, we use Hoeffding's inequality. This is possible as $T_i$ is bounded by $\left[0, \frac 1 s\right]$. We can work out this expression to get a concrete lower bound on the number of samples.
\begin{align*}
\delta &= - \exp(-2 s \epsilon^2 (\E[T'] - w(x))^2)  \\
\log \frac 2 \delta &= 2 s \epsilon^2 (\E[T'] - w(x))^2 \\
s &= \frac 1 {2 \epsilon^2 (\E[T'] - w(x))^2} \log \frac 2 \delta \\
s &> \frac 1 {2 \epsilon^2 \lambda^2} \log \frac 2 \delta = c(\epsilon, \delta) / \lambda^2 
\end{align*}
The last line uses the assumption of the theorem.
\end{proof}

\subsection*{Approximate weighted model sampling}\label{app:approx_wms}

When we use an approximate weighted model sampler for WeightME, we introduce some bias. However, we can still get bounds on the approximation, if the approximate sampler has guarantees. An $(\epsilon, \delta)$-approximate sampler will sample a model $M$ of $\phi$ with a probability between $p(M)/\WMC(\phi)(1-\epsilon)$ and $ p(M)/\WMC(\phi)(1+\epsilon)$, or fail to sample a model with probability less than $\delta$. So in other words, WeightME with an $(\epsilon, \delta)$-approximate sampler is bounded between $(1-\epsilon)\E[X]$ and $(1+\epsilon)\E[X]$, and the bias on the gradient is at most $\epsilon \cdot \grad \log \WMC(\phi)$. Moreover, we retain the probabilistic guarantee of Theorem~\ref{thm:wms}.

\section{Extra Baselines}\label{app:baselines}

The CatLog-trick \cite{de_smet_differentiable_2023}
applies the same decomposition as Theorem~\ref{thm:equiv} to create gradient estimators:
\[
\frac {\partial \WMC(\phi)} {\partial w(x)} = \WMC(\phi \mid x) - \WMC(\phi \mid \neg x)
\]
We can in principle use any other approximate method to compute $\WMC(\phi \vert x)$ and $\WMC(\phi \vert \neg x)$. When these are unbiased, it can be seen as a Rao-Blackwellization.
This method does introduce a linear increase in complexity in the number of variables. In Figure ~\ref{fig:optim}, we include the CatLog-trick in combination with the product t-norm and Gumbel-softmax estimator.

Semantic strengthening needs to compare all clause pairs in order to choose which pairs to conjoin exactly. So it has a quadratic complexity in the number of clauses. As this becomes problematic on large formulas, we also consider the combination of sampling with fuzzy t-norms. Here, we keep sampling as long as we have models. The remaining conjunctions are calculated with the t-norm.

\section{Additional experiments}\label{app:experiments}

\begin{figure*}
    \centering
    \centerline{\includegraphics[width=0.9\textwidth]{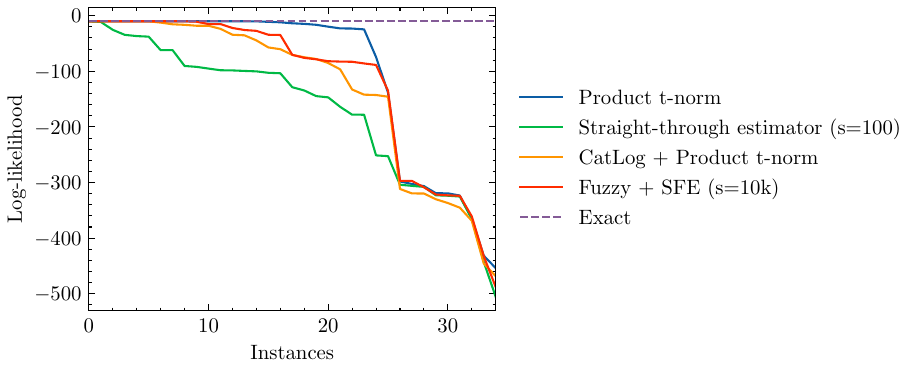}}
    \caption{Maximum negative log-likelihood achieved by the various biased gradient approximations, sorted from best to worst. The benchmarks are 33 easy instances from the Model Counting Competition. Higher is better. All weights are initialized such that 90\% of weights are already correct for a certain model.}
    \label{fig:optim_concept_supervision}
\end{figure*}

In Figure~\ref{fig:optim_concept_supervision}, we repeat the optimization experiment. However, we now provide $90\%$ accurate concept supervision. This means that we pick a model for each formula, and set the weight correctly for $90\%$ of the variables. Although the results are a clear improvement on Figure~\ref{fig:optim}, still none of the methods manage to optimize across the board. This confirms the result in Theorem~\ref{thm:intract_concept}. 

\end{document}